\documentclass{article}

\newif\ifarxiv
\arxivtrue 

\newif\ifneurips
\ifarxiv
	\neuripsfalse
\else
	\neuripstrue
\fi	

\usepackage{our_style}
\usepackage[capitalize]{cleveref}
\usepackage{natbib}

\ifarxiv
	\usepackage[nonatbib,preprint]{neurips_2020}
\else
	\usepackage[nonatbib]{neurips_2020}
\fi

\usepackage[utf8]{inputenc} 
\usepackage[T1]{fontenc}    
\usepackage{url}            
\usepackage{booktabs}       
\usepackage{amsfonts}       
\usepackage{nicefrac}       
\usepackage{microtype}      
\usepackage{algorithm}
\usepackage{algpseudocode}
\usepackage{algorithmicx}
\usepackage{wrapfig}

\renewcommand{\IJ}[1][]{\mathrm{IJ}_{\backslash n #1}}

\title{Approximate Cross-Validation with Low-Rank Data in High Dimensions}

\author{%
  William T. Stephenson \\
  MIT\\
  \texttt{wtstephe@mit.edu} \\
  \And
  Madeleine Udell \\
  Cornell University \\
  \texttt{udell@cornell.edu}
  \And
  Tamara Broderick \\
  MIT\\
  \texttt{tamarab@mit.edu}
}

\begin{document}

\maketitle

\begin{abstract}
Many recent advances in machine learning are driven by a challenging trifecta:
large data size $N$; high dimensions; and expensive algorithms. In this setting, cross-validation (CV) serves as an important tool for model assessment. Recent advances in approximate cross validation (ACV) provide accurate approximations to CV with only a single model fit, avoiding traditional CV's requirement for repeated runs of expensive algorithms. Unfortunately, these ACV methods can lose both speed and accuracy in high dimensions --- unless sparsity structure is present in the data. Fortunately, there is an alternative type of simplifying structure that is present in most data: approximate low rank (ALR). 
Guided by this observation, we develop a new algorithm for ACV that is fast and accurate in the presence of ALR data. 
Our first key insight is that the Hessian matrix --- whose inverse forms the computational bottleneck of existing ACV methods --- is ALR. We show that, despite our use of the \emph{inverse} Hessian, a low-rank approximation using the largest (rather than the smallest) matrix eigenvalues enables fast, reliable ACV. 
Our second key insight is that, in the presence of ALR data, error in existing ACV methods roughly grows with the (approximate, low) rank rather than with the (full, high) dimension. These insights allow us to prove theoretical guarantees on the quality of our proposed algorithm --- along with fast-to-compute upper bounds on its error. We demonstrate the speed and accuracy of our method, as well as the usefulness of our bounds, on a range of real and simulated data sets.
\end{abstract}

\section{Introduction} \label{sec:introduction}

Recent machine learning advances are driven at least in part by increasingly rich data sets --- large in both data size $N$ and dimension $D$. 
The proliferation of data and algorithms makes cross-validation (CV)
\citep{stone:1974:earlyCV, geisser:1975:earlyCV,musgrave:2020:realityCheck}
an appealing tool for model assessment due its ease of use and wide applicability. 
For high-dimensional data sets, leave-one-out CV (LOOCV) is often especially accurate
as its folds more closely match the true size of the data \citep{burman:1989:cvExperiments};
see also Figure 1 of \citet{rad:2018:detailedALOO}. 
Traditionally many practitioners nonetheless avoid LOOCV due its computational expense; it requires
re-running an expensive machine learning algorithm $N$ times.
To address this expense, a number of authors have proposed approximate cross-validation (ACV) methods
\citep{beirami:2017:firstALOO, rad:2018:detailedALOO, giordano:2018:ourALOO}; these methods
are fast to run on large data sets, and both theory and experiments demonstrate their accuracy.
But these methods struggle in high-dimensional problems in two ways.
First, they require inversion of a $D \times D$ matrix, a computationally expensive undertaking.
Second, their accuracy degrades in high dimensions;
see Fig.~1 of \citet{stephenson:2020:sparseALOO} for a classification example
and \cref{highDAccuracy} below for a count-valued regression example.
\citet{koh:2017:influenceFunctions,lorraine:2020:hyperparamOpt} have investigated approximations to the 
matrix inverse for problems similar to ACV, but these approximations do not work well for ACV itself; see \citep[Appendix B]{stephenson:2020:sparseALOO}.
\citet{stephenson:2020:sparseALOO} demonstrate how a practitioner might avoid these high-dimensional problems in the presence of sparse data.
But sparsity may be a somewhat limiting assumption.

\begin{wrapfigure}[23]{r}{0.55\textwidth}
\vspace{-0.7cm}
\centering
\includegraphics[scale=0.5]{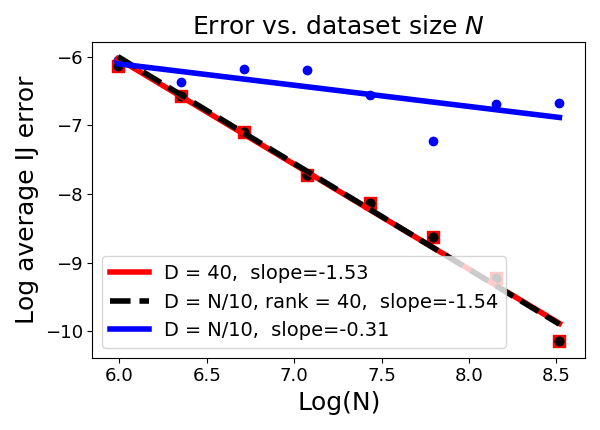}
\caption{Accuracy of the IJ approximation in \cref{IJdef} for a synthetic Poisson regression problem versus the dataset size $N$. Red shows the accuracy when the data dimension is fixed at $D = 40$, blue when the dimension grows as $D = N / 10$, and black when the dimension grows as $D = N / 10$ but with a fixed rank of 40. High-dimensional yet low-rank data has identical performance to low-dimensional data.}
\label{highDAccuracy}
\end{wrapfigure}

We here consider approximately \emph{low-rank} (ALR) data. \citet{udell:2019:bigDataLowRank} argue that
ALR data matrices are pervasive in applications ranging from fluid dynamics and genomics to
social networks and medical records --- and that there are theoretical reasons to expect ALR structure
in many large data matrices. 
For concreteness and to facilitate theory, we focus on fitting generalized linear models (GLMs). We note that GLMs
are a workhorse of practical data analysis; as just one example, the lme4 package \citep{bates:2015:lme4} has been
cited 24,923 times since 2015 as of this writing. 
While accurate ACV methods for GLMs alone thus have potential
for great impact, we expect many of our insights may extend beyond both GLMs and LOOCV (i.e.\ to other CV and bootstrap-like ``retraining'' schemes).

In particular, we propose an algorithm for fast, accurate ACV for GLMs with high-dimensional covariate matrices
--- and provide computable upper bounds on the error of our method relative to exact LOOCV.
Two major innovations power our algorithm.
First, we prove that existing ACV methods automatically obtain high accuracy
in the presence of high-dimensional yet ALR data.
Our theory provides cheaply computable upper bounds on the error of existing ACV methods.
Second, we notice that the $D \times D$ matrix that needs to be inverted in ACV
is ALR when the covariates are ALR.
We propose to use a low-rank approximation to this matrix.
We provide a computable upper bound on the extra error introduced by using such a low-rank approximation.
By studying our bound, we show the surprising fact that, for the purposes of ACV,
the matrix is well approximated by using its \emph{largest} eigenvalues,
despite the fact that ACV uses the matrix inverse.
We demonstrate the speed and accuracy of both our method and bounds with a range of experiments.

\section{Background: approximate CV methods}

We consider fitting a generalized linear model (GLM) with parameter $\theta \in \R^D$ to some dataset
with $N$ observations, $\set{x_n, y_n}_{n=1}^N$,
where $x_n \in \R^D$ are covariates and $y_n \in \R$ are responses.
We suppose that the $x_n$ are approximately low rank (ALR); that is,
the matrix $X \in \R^{N \times D}$ with rows $x_n$ has many singular values near zero.
These small singular values can amplify noise in the responses.
Hence it is common to use $\ell_2$ regularization to ensure that our estimated parameter $\hat\theta$
is not too sensitive to the subspace with small singular values;
the rotational invariance of the $\ell_2$ regularizer automatically penalizes
any deviation of $\theta$ away from the low-rank subspace \citep[Sec.\ 3.4]{hastie:2009:els}.
Thus we consider:
\begin{equation}
	\hat\theta := \argmin_{\theta\in\R^D} \frac{1}{N} \sum_{n=1}^N f(x_n^T\theta, y_n) + \frac{\lambda}{2} \n{\theta}_2^2,
\end{equation}
where $\lambda \geq 0$ is some regularization parameter and $f: \R \times \R \to \R$ is convex in its first argument for each $y_n$.
Throughout, we assume $f$ to be twice differentiable in its first argument.
To use leave-one-out CV (LOOCV), we compute $\thetan$,
the estimate of $\theta$ after deleting the $n$th datapoint from the sum, for each $n$.
To assess the out-of-sample error of our fitted $\hat\theta$, we then compute:
\begin{equation}
	\frac{1}{N} \sum_{n=1}^N \mathrm{Err}(x_n^T\thetan, y_n),
\label{exactLOO}
\end{equation}
where $\mathrm{Err}: \R \times \R \to \R$ is some function measuring the discrepancy
between the observed $y_n$ and its prediction based on $\thetan$ --- 
for example, squared error or logistic loss.

Computing $x_n^T \thetan$ for every $n$ requires solving $N$ optimization problems,
which can be a prohibitive computational expense.
Approximate CV (ACV) methods aim to alleviate this burden via one of two principal approaches below.
Denote the Hessian of the objective by $H := (1/N) \sum_{n=1}^N \nabla^2_\theta f(x_n^T\hat\theta, y_n) + \lambda I_D$
and the $k$th scalar derivative of $f$ as
$\hat D^{(k)}_n := d^{k} f(z, y_n) / dz^k |_{z = x_n^T \hat\theta}$.
Finally, let $Q_n := x_n^T H\inv x_n$ be the $n$th quadratic form on $H\inv$.
The first approximation, based on taking a Newton step from $\hat\theta$
on the objective $(1/N) \sum_{m\neq n}^N f(x_m^T\theta, y_m) + \lambda\|\theta\|_2^2$,
 was proposed by \citet{obuchi:2016:linearALOO,obuchi:2018:logisticALOO,rad:2018:detailedALOO,beirami:2017:firstALOO}.
 We denote this approximation by $\NS$; specializing to GLMs, we have:
\begin{equation}
	x_n^T\thetan \approx x_n^T \NS := x_n^T \hat\theta + \frac{\dnone}{N} \frac{Q_n}{1 - \dntwo Q_n}. \label{NSdef}
\end{equation}
Throughout we focus on discrepancy between $x_n^T \NS$ and $x_n^T\thetan$, rather than between $\NS$ and $\thetan$, since $x_n^T\thetan$ is the argument of \cref{exactLOO}.
See \cref{app:NSIJDerivations} for a derivation of \cref{NSdef}.
The second approximation we consider is based on the infinitesimal jackknife \citep{jaeckel:1972:infinitesimal,efron:1982:jackknife};
it was conjectured as a possible ACV method by \citet{koh:2017:influenceFunctions}, used in a comparison by \citet{beirami:2017:firstALOO},
and studied in depth by \citet{giordano:2018:ourALOO}.
We denote this approximation by $\IJ$; specializing to GLMs, we have:
\begin{equation}
	x_n^T\thetan \approx x_n^T \IJ := x_n^T \hat\theta + (\dnone / N) Q_n . \label{IJdef}
\end{equation}
See \cref{app:NSIJDerivations} for a derivation. We consider both $\NS$ and $\IJ$ in what follows as the two have complementary strengths.
In our experiments in \cref{sec:experiments}, $\NS$ tends to be more accurate;
we suspect that GLM users should generally use $\NS$.
On the other hand, $\NS$ requires the inversion of a different $D \times D$ matrix for each $n$.
In the case of LOOCV for GLMs, each matrix differs by a rank-one update,
so standard matrix inverse update formulas allow us to derive \cref{NSdef},
which requires only a single inverse across folds.
But such a simplification need not generally hold for models beyond GLMs and data re-weightings beyond LOOCV (such as other forms of CV or the bootstrap).
By contrast, even beyond GLMs and LOOCV,
the IJ requires only a single matrix inverse for all $n$.

In any case, we notice that existing theory and experiments for both $\NS$ and $\IJ$
tend to either focus on low dimensions or show poor performance in high dimensions;
see \cref{app:previousTheory} for a review.
One problem is that error in both approximations can grow large in high dimensions.
See
\citep{stephenson:2020:sparseALOO} for an example; also, in \cref{highDAccuracy}, 
we show the $\IJ$ on a synthetic Poisson regression task.
When we fix $D = 40$ and $N$ grows, the error drops quickly;
however, if we fix $D / N = 1/10$ the error is substantially worse.
A second problem is that both $\NS$ and $\IJ$ rely on the computation of $Q_n = x_n^T H\inv x_n$,
which in turn relies on computation\footnote{In practice, for numerical stability, we compute a factorization of $H$ so that $H^{-1} x_n$ can be quickly evaluated for all $n$. However, for brevity, we refer to computation of the inverse of $H$ throughout.} of $H\inv$.
The resulting $O(D^3)$ computation time quickly becomes impractical in high dimensions.
Our major contribution is to show that both of these issues can be avoided when the data are ALR.

\section{Methodology}
\label{sec:methodology}

\begin{algorithm}
    \caption{Approximation to $\{ x_n^T \thetan\}_{n=1}^N$ for low-rank GLMs}
    \label{alg:mainAlgorithm}
    \begin{algorithmic}[1] 
        \Procedure{AppxLOOCV}{$\hat\theta, X, \lambda, \{\dnone\}_{n=1}^N, \{\dntwo\}_{n=1}^N, K$}
        	\State $B \gets X^T \mathrm{diag}\{\dntwo\}_{n=1}^N X$ \Comment{The Hessian, $H$ equals $B + \lambda I_D$}
        	\State $\{ \tildeqn \}_{n=1}^N \gets \Call{AppxQn}{B, K, \lambda}$ \Comment{Uses rank-$K$ decomposition of $B$ (\cref{sec:computation})}
            \For{$n = 1, \dots, N$}
                \State \textbf{either} $x_n^T \tildeNS \gets x_n^T \NS(\tildeqn)$ \Comment{i.e., compute \cref{NSdef} using $\tildeqn$ instead of $Q_n$}
                \State \textbf{or} $x_n^T \tildeIJ \gets x_n^T \IJ(\tildeqn)$ \Comment{i.e., compute \cref{IJdef} using $\tildeqn$ instead of $Q_n$}
            \EndFor\label{euclidendwhile}
            \State \textbf{return} $\{x_n^T \tildeNS\}_{n=1}^N$ \textbf{or} $\{x_n^T \tildeIJ\}_{n=1}^N$ \Comment{User's choice}
        \EndProcedure
    \end{algorithmic}
\end{algorithm}

We now present our algorithm for fast, approximate LOOCV in GLMs with ALR data. We then state our main theorem, which (1) bounds the error in our algorithm
relative to exact CV, (2) gives the computation time of our algorithm, and (3) gives the computation time of our bounds. Finally we discuss the implications of our theorem before moving on to the details of its proof in the remainder of the paper.

Our method appears in \cref{alg:mainAlgorithm}. To avoid the $O(D^3)$ matrix inversion cost, we replace $H$ by $\tildeH \approx H$,
where $\tildeH$ uses a rank-$K$ approximation and can be quickly inverted.
We can then compute $\tildeqn := x_n^T \tildeH\inv x_n \approx Q_n$, which enters into either the
NS or IJ approximation, as desired.

Before stating \cref{thm:mainThm}, we establish some notation.
We will see in \cref{prop:basicErrorBound} of \cref{sec:accuracy}
that we can provide computable upper bounds $\qnbnd \geq \abs{\tildeqn - Q_n}$; $\eta_n$ will enter directly into the error bound for $x_n^T \tildeIJ$ in \cref{thm:mainThm} below.
To bound the error of $x_n^T \tildeNS$, we need to further define
\begin{equation*}
	E_n := \max \set{\left| \frac{\tildeqn + \qnbnd}{1 - \dntwo (\tildeqn + \qnbnd)} - \frac{\tildeqn}{1 - \dntwo \tildeqn} \right|, \; \left| \frac{\tildeqn - \qnbnd}{1 - \dntwo (\tildeqn - \qnbnd)} - \frac{\tildeqn}{1 - \dntwo \tildeqn} \right| }.
\end{equation*}
Additionally, we will see in \cref{prop:MnBound} of \cref{sec:accuracy} that we can bound the ``local Lipschitz-ness'' of the Hessian related to the third derivatives of $f$ evaluated at some $z$, $\dnthree(z) := d^3 f(z,y_n) / dz^3 |_{z = z}$. We will denote our bound by $M_n$:
\begin{equation}
 	M_n \geq \left(\frac{1}{N} \sum_{m\neq n} \n{x_m}_2^2 \right)  \max_{s \in [0,1]} \left| \dnthree\left(x_n^T ((1-s)\hat\theta + s\thetan) \right) \right|,
 	\label{MnDef}
\end{equation}
We are now ready to state, and then discuss, our main result --- which is proved in \cref{app:mainTheoremProof}.
\begin{thm} \label{thm:mainThm}
	\textbf{(1) Accuracy:} Let $\qnbnd \geq \abs{Q_n - \tilde Q_n}$ be the upper bound produced by \cref{prop:basicErrorBound} and $M_n$ the local Lipschitz constants computed in \cref{prop:MnBound}. Then the estimates $x_n^T \tildeNS$ and $x_n^T \tildeIJ$ produced by \cref{alg:mainAlgorithm} satisfy:
	\begin{align}
		& \abs{x_n^T \tildeNS - x_n^T \thetan} \leq \frac{M_n}{N^2 \lambda^3} \abs{\dnone}^2 \n{x_n}_2^3 + \abs{\dnone} E_n
			\label{algNSBnd} \\
		& \abs{x_n^T \tildeIJ - x_n^T \thetan} \leq \frac{M_n}{N^2 \lambda^3} \abs{\dnone}^2 \n{x_n}_2^3 + \frac{1}{N^2 \lambda^2} \abs{\dnone} \dntwo \n{x_n}_2^4 + \abs{\dnone} \qnbnd.
			\label{algIJBnd}
	\end{align}
	\textbf{(2) Algorithm computation time:} The runtime of \cref{alg:mainAlgorithm} is in $O(NDK + K^3)$. 
	\textbf{(3) Bound computation time:} The upper bounds in \cref{algNSBnd,algIJBnd} are computable in $O(DK)$ time for each $n$ for common GLMs such as logistic and Poisson regression.
\end{thm}

To interpret the running times, note that standard ACV methods have total runtime in $O(ND^2 + D^3)$. So \cref{alg:mainAlgorithm} represents a substantial speedup when the dimension $D$ is large and $K \ll D$. Also, note that our bound computation time has no worse behavior than our algorithm runtime.
We demonstrate in our experiments (\cref{sec:experiments}) that our error bounds are both computable and useful in practice. 
To help interpret the bounds, note that they contain two sources of error: (A) the error of our additional approximation relative to existing ACV methods (i.e.\ the use of $\tildeqn$) and (B) the error of existing ACV methods in the presence of ALR data. 
Our first corollary notes that (A) goes to zero as the data becomes exactly low rank.
\begin{cor}
	\label{cor:lowRank}
	As the data becomes exactly low rank with rank $R$ (i.e., $X$'s lowest singular values $\sigma_d \to 0$ for $d = R+1, \dots, D$), we have $\eta_n, E_n \to 0$ if $K \geq R$.
\end{cor}
See \cref{app:corLowRank} for a proof. Our second corollary gives an example for which the error in existing (exact) ACV methods (B) vanishes as $N$ grows.
\begin{cor}
	\label{cor:Ninfinity}
	Suppose the third derivatives $\dnthree$ and the $x_n$ are both bounded and the data are exactly low-rank with constant rank $R$. Then with $N \to\infty$, $D$ growing at any rate, and $K$ arbitrary, the right hand sides of \cref{algNSBnd,algIJBnd} reduce to $\abs{\dnone} E_n$ and $\abs{\dnone} \eta_n$, respectively.
\end{cor}
We note that \cref{cor:Ninfinity} is purely illustrative, and we strongly suspect that none of its conditions are necessary. Indeed, our experiments in \cref{sec:experiments} show that the bounds of \cref{thm:mainThm} imply reasonably low error for non-bounded derivatives with ALR data and only moderate $N$.

\section{Accuracy of exact ACV with approximately low-rank data} \label{sec:accuracy} 

Recall that the main idea behind \cref{alg:mainAlgorithm} is to compute a fast approximation to existing ACV methods
by exploiting ALR structure.
To prove our error bounds,
we begin by proving that the exact ACV methods $\NS$ and $\IJ$ approximately (respectively, exactly) retain
the low-dimensional accuracy displayed in red in \cref{highDAccuracy} when applied to GLMs
with approximately (respectively, exactly) low-rank data.
We show the case for exactly low-rank data first.
Let $X = U\Sigma V^T$ be the singular value decomposition of $X$,
where $U \in \R^{N \times D}$ has orthonormal columns,
$\Sigma \in \R^{D \times D}$ is a diagonal matrix with only $R \ll D$ non-zero entries,
and $V \in \R^{D \times D}$ is an orthonormal matrix.
Let $V_{:R}$ denote the first $R$ columns of $V$,
and fit a model restricted to $R$ dimensions as:
$$
	\hat\phi := \argmin_{\phi \in \R^R} \frac{1}{N} \sum_{n=1}^N f((V_{:R}^T x_n)^T \phi) + \frac{\lambda}{2} \| \phi \|_2^2.
$$
Let $\restrictedThetan$ be the $n$th leave-one-out parameter estimate from this problem, and let $\restrictedNS$ and $\restrictedIJ$ be \cref{NSdef} and \cref{IJdef} applied to this restricted problem. We can now show that the error of $\IJ$ and $\NS$ applied to the full $D$-dimensional problem is exactly the same as the error of $\restrictedNS$ and $\restrictedIJ$ applied to the restricted $R \ll D$ dimensional problem.
\begin{lm} \label{prop:exactLowRankAccuracy}
	Assume that the data matrix $X$ is exactly rank $R$.
	Then 
	$
		\abs{ x_n^T \NS - x_n^T \thetan} = \abs{ (V_{:R}^T x_n)^T \restrictedNS - (V_{:R}^T x_n)^T \restrictedThetan}
	$ 
	and 
	$	
		\abs{ x_n^T \IJ - x_n^T \thetan} = \abs{(V_{:R}^T x_n)^T \restrictedIJ - (V_{:R}^T x_n)^T \restrictedThetan}
	$.
\end{lm}
See \cref{app:ELRAccuracy} for a proof.
Based on previous work (e.g., \citep{beirami:2017:firstALOO, rad:2018:detailedALOO, giordano:2018:ourALOO}),
we expect the ACV errors $\abs{(V_{:R}^T x_n)^T \restrictedNS - (V_{:R}^T x_n)^T \restrictedThetan}$ and $\abs{ (V_{:R}^T x_n)^T \restrictedIJ - (V_{:R}^T x_n)^T \restrictedThetan}$ to be small,
as they represent the errors of $\NS$ and $\IJ$ applied to an $R$-dimensional problem.
We confirm \cref{prop:exactLowRankAccuracy} numerically in \cref{highDAccuracy}, where the error for the $D = 40$ problems (red) exactly matches that of the high-$D$ but rank-40 problems (black).

However, real-world covariate matrices $X$ are rarely exactly low-rank. By adapting results from \citet{wilson:2020:modelSelectionALOO}, we can give bounds that smoothly decay as we leave the exact low-rank setting of \cref{prop:exactLowRankAccuracy}. 
To that end, define:
\begin{equation}
 L_n := \left( \frac{1}{N} \sum_{m: \, m\neq n}^N \n{x_m}_2^2 \right) \max_{s \in [0,1]} \dnthree\left(x_n^T ((1-s)\hat\theta + s\thetan) \right).
 \label{LnDef}
\end{equation}
\begin{lm} \label{prop:appxLowRankAccuracy}
	Assume that $\lambda > 0$. Then, for all $n$:
	\begin{align}
		& \abs{x_n^T \NS - x_n^T \thetan} \leq \frac{L_n}{N^2 \lambda^3} \abs{\dnone}^2 \n{x_n}_2^3 \label{appxBoundNS} \\
		& \abs{x_n^T \IJ - x_n^T \thetan} \leq \frac{L_n}{N^2 \lambda^3} \abs{\dnone}^2 \n{x_n}_2^3 + \frac{1}{N^2 \lambda^2} \abs{\dnone} \dntwo \n{x_n}_2^4. \label{appxBoundIJ}
	\end{align}
	Furthermore, these bounds continuously decay as the data move from exactly to approximately low rank in that they are continuous in the singular values of $X$.
\end{lm}
The proofs of \cref{appxBoundNS,appxBoundIJ} mostly follow from results in \citet{wilson:2020:modelSelectionALOO},
although our results removes a Lipschitz assumption on the $\dntwo$; see \cref{app:ALRAccuracy} for a proof.


Our bounds are straightforward to compute; we can calculate
the norms $\| x_n \|_2$ and evaluate
the derivatives $\dnone$ and $\dntwo$ at the known $x_n^T \hat\theta$.
The only unknown quantity is $L_n$.
However, we can upper bound the $L_n$ using the following proposition.
\begin{prop} \label{prop:MnBound}
	Let $\mathcal{Z}_n$ be the set of $z \in \R$ such that $\abs{z} \leq \abs{x_n^T \hat\theta} + \abs{\dnone} \| x_n\|_2^2 / (N\lambda)$. For $L_n$ as defined in \cref{LnDef}, we have the upper bound:
	\begin{equation}
		L_n \leq M_n := \max_{z \in \mathcal{Z}_n} \abs{\dnthree(z)} \left( \frac{1}{N} \sum_{m: \, m\neq n}^N \n{x_m}_2^2 \right).
	\end{equation}
\end{prop}
To compute an upper bound on the $M_n$ in turn,
we can optimize $\dnthree(z)$ for $\abs{z} \leq \abs{x_n^T \hat\theta} + \abs{\dnone} \| x_n\|_2^2 / (N\lambda)$.
This scalar problem is straightforward for common GLMs:
for logistic regression, we can use the fact that $\abs{\dnthree} \leq 1/4$,
and for Poisson regression with an exponential link function
(i.e., $y_n \sim \mathrm{Poisson}(\exp(x_n^T \theta))$),
we maximize $\dnthree(z) = e^z$ with the largest $z \in \mathcal{Z}_n$.

\section{Approximating the quadratic forms $Q_n$} \label{sec:computation}
\mnote{Would be nice to find a good name for $Q_n$ to use throughout.
It's particularly unpleasant to use $Q_n$ (rather than its name) in a section header.
But I can't think of a good name\ldots
It measures $x_n$ in the inverse Hessian norm\ldots
It's really a measure for how much of an outlier $x_n$ is\ldots}

\begin{algorithm}
    \caption{Estimate $Q_n = x_n^T (B + \lambda I_D)\inv x_n$ via a rank-$K$ decomposition of PSD matrix $B$. \emph{Note}: as written, this procedure is not numerically stable. See \cref{app:computationImplementation} for an equivalent but numerically stable version.}
    \label{alg:customizedTropp}
    \begin{algorithmic}[1] 
        \Procedure{AppxQn}{$B,K,\lambda$}
            \For{$k = 1, \dots, K$}
        		\State $\mathcal{E}_k \gets \normaldist (0_{D}, I_{D})$ \Comment{$\mathcal{E} \in \R^{D \times K}$ has i.i.d. $\normaldist(0,1)$ entries}
        	\EndFor
        	\State $\Omega \gets \Call{OrthonormalizeColumns}{ \mathrm{diag}\{1/(B_{dd} + \lambda)\}_{d=1}^D X^T X \mathcal{E} }$ \label{line:omegaChoice} \Comment{\cref{prop:optimalAgree}}
        	\State $M \gets B\Omega$
        	\State $\tildeH \gets M (\Omega^T M)\inv M^T + \lambda I_D$ \Comment{Rank-$K$ Nystr\"{o}m approximation of $B$}
        	 \For{$n = 1, \dots, N$}
                \State $\tildeqn \gets \mathrm{min} \set{ x_n^T \tildeH\inv x_n, \; \n{x_n}_2^2 / (\lambda + \dntwo \n{x_n}_2^2)}$ \Comment{\cref{prop:qnBound}}
           \EndFor
        \State \textbf{return} $\{ \tildeqn \}_{n=1}^N$
        \EndProcedure
    \end{algorithmic}
\end{algorithm}

The results of \cref{sec:accuracy} imply that existing ACV methods achieve high accuracy
on GLMs with ALR data.
However, in high dimensions, the $O(D^3)$ cost of computing $H\inv$ in the $Q_n$
can be prohibitive.
\citet{koh:2017:influenceFunctions,lorraine:2020:hyperparamOpt} have investigated an approximation to the 
matrix inverse for problems similar to ACV; however, in our experiments in \cref{app:neumannNoGood}, we find that this method does not work well for ACV.
Instead, we give approximations $\tildeqn \approx Q_n$
in \cref{alg:customizedTropp} along with computable upper bounds on the error $\abs{\tildeqn - Q_n}$ in \cref{prop:qnBound}.
When the data has ALR structure, so does the Hessian $H$;
hence we propose a low-rank matrix approximation to $H$.
This gives \cref{alg:customizedTropp} a runtime in $O(NDK + K^3)$, which can result in substantial savings
relative to the $O(ND^2 + D^3)$ time required to exactly compute the $Q_n$.
We will see that the main insights behind \cref{alg:customizedTropp} come from studying an upper bound on the approximation error when using a low-rank approximation.

Observe that by construction of $\Omega$ and $\tildeH$ in \cref{alg:customizedTropp}, the approximate Hessian $\tildeH$
exactly agrees with $H$ on the subspace $\Omega$.
\mnote{Might be nice to mention this closer to \cref{alg:customizedTropp}.}
We can compute an upper bound on the error $\abs{x_n^T \tildeH\inv x_n - Q_n}$
by recognizing that any error originates from components of $x_n$ orthogonal to $\Omega$:
\begin{prop} \label{prop:basicErrorBound}
	Let $\lambda > 0$ and suppose there is some subspace $\mathcal{B}$ on which $H$ and $\tildeH$ exactly agree:
  $\forall v \in \mathcal B, Hv = \tildeH v$.
  Then $H\inv$ and $\tildeH\inv$ agree exactly on the subspace $\agree := H \mathcal{B}$, and
	\begin{equation}
		\abs{x_n^T \tildeH\inv x_n - Q_n} \leq \frac{\n{P^\perp_\agree x_n}_2^2}{\lambda},
    \qquad \text{for all~} n = 1,\dots, N,
	\label{basicErrorBound}
	\end{equation}
	where $P^\perp_\agree$ denotes projection onto the orthogonal complement of $\mathcal{A}$.
\end{prop}
For a proof, see \cref{app:computationProofs}.
The bound from \cref{basicErrorBound} is easy to compute in $O(DK)$ time given a basis for $\fwdAgree$.
It also motivates the choice of $\Omega$ in \cref{alg:customizedTropp}.
In particular, \cref{prop:optimalAgree} shows that $\Omega$ approximates
the rank-$K$ subspace $\fwdAgree$ that minimizes the average of the bound in \cref{basicErrorBound}.
\begin{prop} \label{prop:optimalAgree}
	Let $V_{:K} \in \R^{D \times K}$ be the matrix with columns equal to the right singular vectors of $X$ corresponding to the $K$ largest singular values. Then the rank-$K$ subspace $\fwdAgree$ minimizing $\sum_n \| P^\perp_\agree x_n\|_2^2$ is an orthonormal basis for the columns of $H\inv V_{:K}$.
\end{prop}
\begin{proof}
	$\sum_n \| P^\perp_\agree x_n\|_2^2 = \| (I_D - P_\agree) X^T \|_F^2$, where $\|\cdot\|_F$ denotes the Frobenius norm. Noting that the given choice of $\fwdAgree$ implies that $\agree = V_{:K}$, the result follows from the Eckart-Young theorem.
\end{proof}
We can now see that the choice of $\Omega$ in \cref{alg:customizedTropp} approximates
the optimal choice $H\inv V_{:K}$.
In particular, we use a single iteration of the subspace iteration method \citep{bathe:1973:subspaceIteration}
to approximate $V_{:K}$ and then multiply by
the diagonal approximation $\mathrm{diag}\set{1 / H_{dd}}_{d=1}^D \approx H\inv$.
This approximation uses the top singular vectors of $X$.
We expect these directions to be roughly equivalent to
the largest eigenvectors of $B := \sum_n \dntwo x_n x_n^T$,
which in turn are the largest eigenvectors of $H = B + \lambda I_D$.
Thus we are roughly approximating $H$ by its \emph{largest} $K$ eigenvectors.

Why is it safe to neglect the small eigenvectors?
At first glance this is strange, as to minimize the operator norm $\| H\inv - \tildeH\inv\|_{op}$,
one would opt to preserve the action of $H$ along its \emph{smallest} $K$ eigenvectors.
The key intuition behind this reversal is that we are interested in the action of $H\inv$
in the direction of the datapoints $x_n$,
which, on average, tend to lie near the largest eigenvectors of $H$.

\cref{alg:customizedTropp} uses one additional insight to improve the accuracy of its estimates.
In particular, we notice that, by the definition of $H$, each $x_n$ lies in an eigenspace of $H$ with eigenvalue at least $\dntwo \| x_n \|_2^2 + \lambda$.
This observation undergirds the following result,
which generates our final estimates $\tildeqn \approx Q_n$,
along with quickly computable bounds on their error $\abs{\tildeqn - Q_n}$.
\begin{prop} \label{prop:qnBound}
	The $Q_n = x_n^T H\inv x_n$ satisfy
	$
		0 < Q_n \leq \|x_n\|_2^2 / (\lambda + \dntwo \| x_n \|_2^2).
	$
	Furthermore, letting $\tildeqn := \mathrm{min} \{ x_n^T \tildeH\inv x_n, \, \n{x_n}_2^2 / (\lambda + \dntwo \n{x_n}_2^2) \}$, we have the error bound
	\begin{equation}
		\abs{\tildeqn - Q_n} \leq \min\set{\frac{\n{P^\perp_\agree x_n}_2^2}{\lambda}, \frac{\|x_n\|_2^2}{\lambda + \dntwo \|x_n\|_2^2}}.
		\label{qntildeBound}
	\end{equation}
\end{prop}
See \cref{app:computationProofs} for a proof. We finally note that \cref{alg:customizedTropp} strongly resembles algorithms from the randomized numerical linear algebra literature. Indeed, the work of \citet{tropp:2017:psdRandomizedApproximation} was the original inspiration for \cref{alg:customizedTropp}, and \cref{alg:customizedTropp} can be seen as an instance of the algorithm presented in \citet{tropp:2017:psdRandomizedApproximation} with specific choices of various tuning parameters optimized for our application. For more on this perspective, see \cref{app:tropp}.

\section{Experiments} \label{sec:experiments}

\begin{figure}
\centering
\begin{tabular}{cc}
	\includegraphics[scale=0.45]{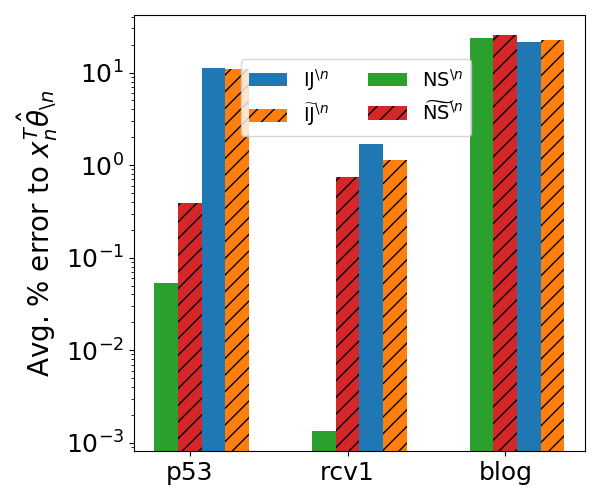} &
	\includegraphics[scale=0.45]{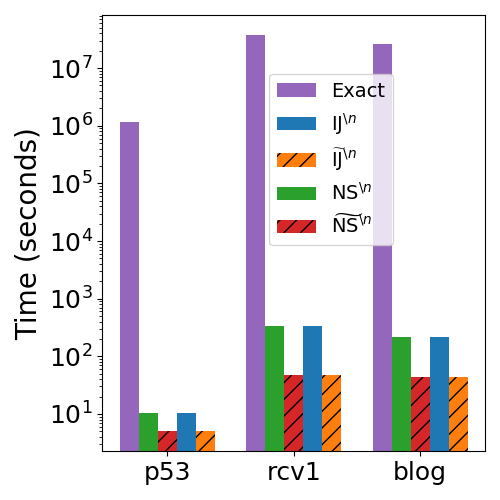}
\end{tabular}
\caption{Experiments on real datasets. (\emph{Left}): average percent error compared to exact CV on a subset of datapoints, $(1/20) \sum_{b=1}^{20} \abs{ x_b^T \mathrm{approx.} - x_b^T \hat\theta_{\backslash b}} / \abs{x_b^T \hat\theta_{\backslash b}}$, where $\mathrm{approx.}$ denotes $\NS, \tildeNS, \IJ,$ or $\tildeIJ$.
 (\emph{Right}): ACV runtimes with exact CV runtimes for comparison. ACV runtimes are given for all $N$ datapoints. Exact CV runtimes are estimated runtimes for all $N$ datapoints.}
 \label{fig:realDataExperiments}
\end{figure}

\paragraph{\cref{alg:mainAlgorithm} on real data.}
We begin by confirming the accuracy and speed of \cref{alg:mainAlgorithm} on real data compared to both exact CV and existing ACV methods.
We apply logistic regression to two datasets (\texttt{p53} and \texttt{rcv1})
and Poisson regression to one dataset (\texttt{blog}).
\texttt{p53} has a size of roughly $N = 8{,}000, D = 5{,}000$, and the remaining two have roughly $N = D = 20{,}000$; see \cref{app:realDataExperiments} for more details.
We run all experiments with $\lambda = 5.0$.
To speed up computation of exact LOOCV, we only run over twenty randomly chosen datapoints.
We report average percent error, $ (1/20) \sum_{b=1}^{20} \abs{x_b^T \mathrm{appx.} - x_b^T \hat\theta_{\backslash b}} / \abs{x_b^T \hat\theta_{\backslash b}}$
for each exact ACV algorithm and the output of \cref{alg:mainAlgorithm}.
For the smaller dataset \texttt{p53},
the speedup of \cref{alg:mainAlgorithm} over exact $\NS$ or $\IJ$ is marginal;
however, for the larger two datasets, our methods provide significant speedups:
for the \texttt{blog} dataset, we estimate the runtime of full exact CV to be nearly ten months.
By contrast, the runtime of $\NS$ is nearly five minutes,
and the runtime of $\tildeNS$ is forty seconds.
In general, the accuracy of $\tildeIJ$ closely mirrors the accuracy of $\IJ$,
while the accuracy of $\tildeNS$ can be somewhat worse than that of $\NS$
on the two logistic regression tasks;
however, we note that in these cases, the error of $\tildeNS$ is still less than 1\%.

\begin{figure}
\centering
\begin{tabular}{cc}
	\includegraphics[scale=0.45]{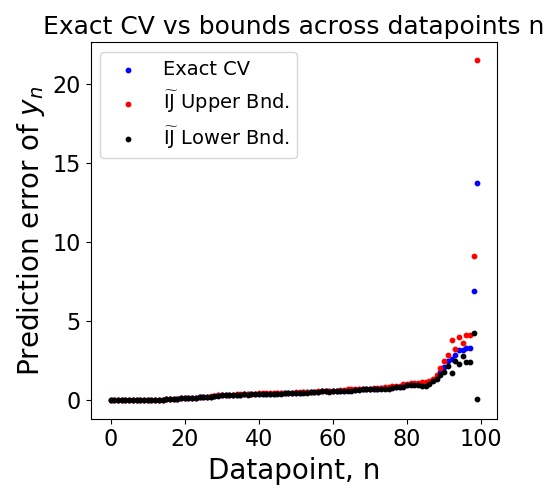} &
	\includegraphics[scale=0.45]{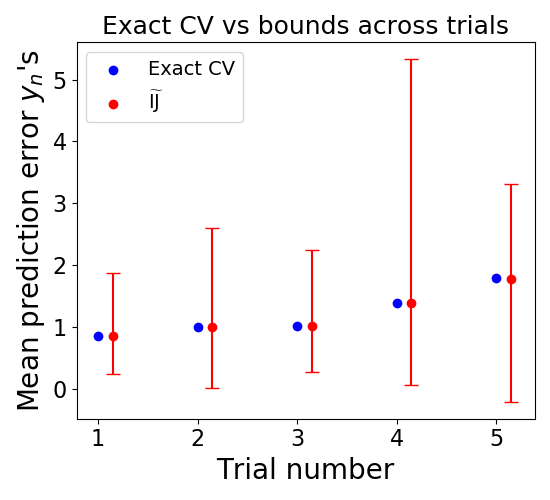}
\end{tabular}
\caption{Error bounds implied by \cref{thm:mainThm} on $\tildeIJ$'s estimate of out-of-sample error using squared loss, $\mathrm{Err}(x_n^T\theta, y_n) = (e^{x_n^T\theta} - y_n)^2$. (\emph{Left}): Per datapoint error bounds. The bound is fairly loose for datapoints with larger squared loss, but tighter for points with lower squared loss. (\emph{Right}): Five trials with estimates averaged across all $n$. We compute the upper (respectively, lower) error bars for $\tildeIJ$ by averaging the upper (respectively, lower) bounds. While our bounds overstate the difference between exact CV and $\tildeIJ$, they still give non-vacuous information on value of exact CV.}
\label{fig:errorBounds}
\end{figure}

\paragraph{Accuracy of error bounds.}
We next empirically check the accuracy of the error bounds from \cref{thm:mainThm}.
We generate a synthetic Poisson regression problem with i.i.d.\ covariates $x_{nd} \sim \normaldist(0,1)$
and $y_n \sim \poissondist(e^{x_n^T \theta^*})$,
where $\theta^* \in \R^D$ is a true parameter with i.i.d.\ $\normaldist(0,1)$ entries.
We generate a dataset of size $N = 800$ and $D = 500$ with covariates of approximate rank $50$; we arbitrarily pick $\lambda = 1.0$. In \cref{fig:errorBounds}, we illustrate the utility of the bounds from \cref{thm:mainThm}
by estimating the
out-of-sample loss with $\mathrm{Err}(x_n^T \theta, y_n) = (e^{x_n^T\theta} - y_n)^2$.
Across five trials, we show the results of exact LOOCV, our estimates provided by $\tildeIJ$,
and the bounds on the error of $\tildeIJ$ given by \cref{thm:mainThm}. While our error bars in \cref{fig:errorBounds} tend to overestimate the difference between $\tildeIJ$ and exact CV, they typically provide upper bounds on exact CV on the order of the exact CV estimate itself.
In some cases, we have observed that the error bars can overestimate exact CV by many orders of magnitude
(see \cref{app:errorBoundExps}),
but this failure is usually due to one or two datapoints $n$ for which the bound is vacuously large.
A simple fix is to resort to exact CV just for these datapoints.

\section{Conclusions}
 We provide an algorithm to approximate CV accurately and quickly in high-dimensional GLMs with ALR structure. Additionally, we provide quickly computable upper bounds on the error of our algorithm. We see two major directions for future work. First, while our theory and experiments focus on ACV for model assessment, the recent work of \citet{wilson:2020:modelSelectionALOO} has provided theoretical results on ACV for model selection (e.g.\ choosing $\lambda$). It would be interesting to see how dimensionality and ALR data plays a role in this setting. Second, as noted in the introduction, we hope that the results here will provide a springboard for studying ALR structure in models beyond GLMs and CV schemes beyond LOOCV.

\ifneurips
	\newpage

\subsection*{Broader Impact}

In general, we feel that work assessing the accuracy of machine learning models will have a positive impact on society. As machine learning is deployed in areas in which mistakes could have adverse effect on peoples' lives, it is important that we understand the error rate of such decisions before deployment. On the other hand, machine learning models can (and are) used for harm and the methods in this paper may assist in the development in such models. Additionally, there is always a risk in introducing any sort of approximation, as it may fail silently and unexpectedly in practice; e.g., our approximations might incorrectly lead a practitioner to conclude that their machine learning model has very small error when the opposite is in fact true. While we believe the computable upper bounds provided here somewhat mitigate this issue, we still remain cautious (though optimistic) about applying ACV methods in practice. Finally, we note that an implicit assumption throughout our work is that computing exact CV is something we want; that is, exact CV provides a good estimate of out-of-sample error. While this seems to be generally true, this does add another failure mode to our algorithm. In particular, even if we provide an accurate approximation to exact CV, it may be that exact CV itself is misleading.
\fi

\ifarxiv
	\subsubsection*{Acknowledgements}
	The authors thank Zachary Frangella for helpful conversations.
WS and TB were supported by the CSAIL-MSR Trustworthy AI Initiative,
an NSF CAREER Award, an ARO Young Investigator Program Award, and ONR
Award N00014-17-1-2072. MU was supported by NSF Award IIS-1943131,
\fi

\bibliography{references}
\bibliographystyle{plainnat}

\newpage
\appendix

\section{Derivation of $x_n^T \NS$ and $x_n^T \IJ$}
\label{app:NSIJDerivations}

Here, we derive the expressions for $x_n^T \NS$ and $x_n^T \IJ$ given in \cref{NSdef,IJdef}. We recall from previous work (e.g., see \citet[Appendix C]{stephenson:2020:sparseALOO} for a summary) that the LOOCV parameter estimates given by the Newton step and infinitesimal jackknife approximations are given by
\begin{align}
	&\thetan \approx \NS := \hat\theta + \frac{1}{N} \left(\sum_{m: \, m\neq n}^N \dntwo x_m x_m^T + \lambda I_D \right)\inv \dnone x_n \\
	&\thetan \approx \IJ := \hat\theta + \frac{1}{N} \left( \sum_{n=1}^N \dntwo x_n x_n^T \lambda I_D \right)\inv \dnone x_n.
\end{align}
Taking the inner product of $\IJ$ with $x_n$ immediately gives \cref{IJdef}. To derive \cref{NSdef}, define $H := \sum_n \dntwo x_n x_n^T + \lambda I_D$ and note that we can rewrite $\NS$ using the Sherman-Morrison formula:
$$
\NS = \hat\theta + \frac{1}{N} \left[ \dnone H\inv x_n + \dnone\dntwo \frac{H\inv x_n x_n^T H\inv}{1 - \dntwo Q_n} x_n \right].
$$
Taking the inner product with $x_n$ and reorganizing gives:
$$
x_n^T \NS = x_n^T \hat\theta + \frac{\dnone}{N} \left[ \frac{Q_n - \dntwo Q_n^2}{1 - \dntwo Q_n} + \frac{\dntwo Q_n^2}{1 - \dntwo Q_n} \right] = x_n^T \hat\theta + \frac{\dnone}{N} \frac{Q_n}{1 - \dntwo Q_n}.
$$
\section{Comparison to existing Hessian inverse approximation}
\label{app:neumannNoGood}

\begin{figure}
	\centering
	\begin{tabular}{cc}
		\includegraphics[scale=0.4]{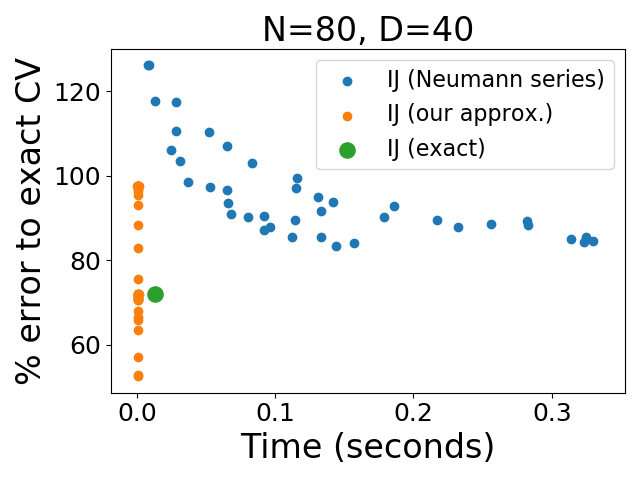} &
		\includegraphics[scale=0.4]{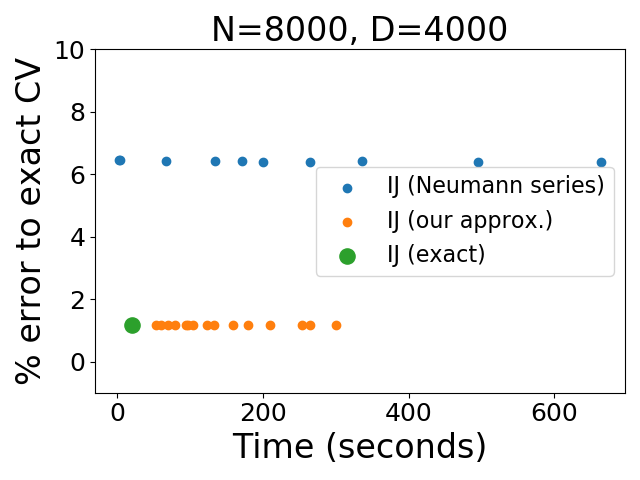} \\
		\includegraphics[scale=0.4]{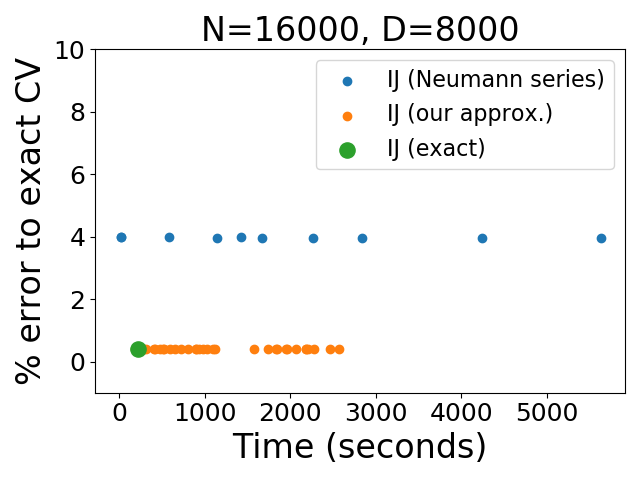} &
		\includegraphics[scale=0.4]{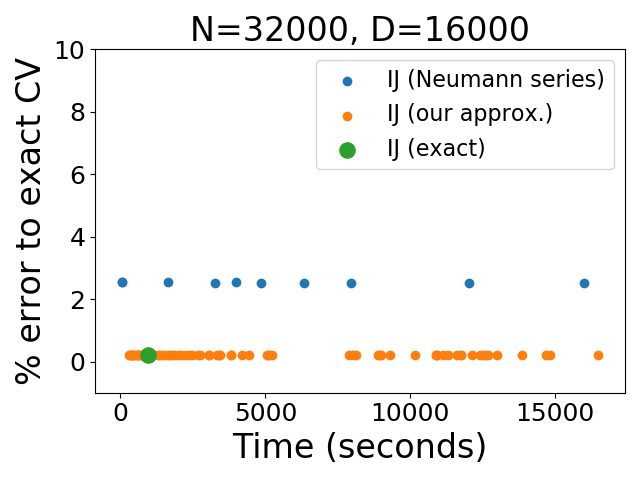}
	\end{tabular}
	\caption{Experiment from \cref{app:neumannNoGood}. Across four different dataset sizes, using the Neumann series approximation (orange) does not show any noticeable improvement on the time scale of running our approximation (green) for all possible values of $K$.}
	\label{fig:neumannNoGood}
\end{figure}

We note that two previous works have used inverse Hessian approximations for applications similar to ACV. \citet{koh:2017:influenceFunctions} use influence functions to estimate behavior of black box models, and \citet{lorraine:2020:hyperparamOpt} use the implicit function theorem to optimize model hyperparameters. In both papers, the authors need to multiply an inverse Hessian by a gradient. To deal with the high dimensional expense associated with this matrix inverse, both sets of authors use the method of \citet{agarwal:2016:lissa}, who propose a stochastic approximation to the \emph{Neumann series}. The Neumann series writes the inverse of a matrix $H$ with operator norm $\|H \|_{op} < 1$ as:
$$
	H\inv = \sum_{k=0}^\infty (I - H)^k.
$$
The observation of \citet{agarwal:2016:lissa} is that this series can be written recursively, as well as estimated stochastically if one has random variables $A_s$ with $\mathbb{E}[A_s] = H$. In the general case of empirical risk minimization with an objective of $(1/N) \sum_{n=1}^N f_n(\theta)$, \citet{agarwal:2016:lissa} propose using $A_s = \nabla^2 f_s(\theta)$ for some $s \in [N]$ chosen uniformly at random. In the GLM setting we are interested in here, we choose an index $s \in [N]$ uniformly at random and set $A_s = \dntwo x_s x_s^T + (\lambda / N) I_D$. Then, for $s = 1, \dots, S$, we follow \citet{agarwal:2016:lissa} to recursively define:
$$
	H\inv \approx \bar H_{s}\inv := I_D + (I - A_s) \bar H_{s-1}\inv.
$$
The final recommendation of \citet{agarwal:2016:lissa} is to repeat this process $M$ times and average the results. We thus have two hyperparameters to choose: $S$ and $M$. 

To test out the \citet{agarwal:2016:lissa} approximation against our approximation in \cref{alg:mainAlgorithm}, we generate Poisson regression datasets of increasing sizes $N$ and $D$. We generate approximately low-rank covariates $x_n$ by drawing $x_{nd} \sim N(0, 1)$ for $d = 1, \dots, 1{,}000$ and $x_{nd} \sim N(0, 0.01)$ for $d = 1{,}001, \dots, D$; for our dataset with $D = 40$, we follow the same procedure but with $R = 20$ instead. 
For each dataset, we compute $\IJ$, as well as our approximation $\tildeIJ$ from \cref{alg:mainAlgorithm}. 
We run \cref{alg:mainAlgorithm} for $K = 1, 100, 200, \dots, D$ and run the stochastic Neumann series approximation with all combinations of $M \in \set{2, 5}$ and $S \in \set{1, 5, 10, 15, \dots, 200}$.
We measure the accuracy of all approximations as percent error to exact CV ($x_n^T \thetan$).
We show in \cref{fig:neumannNoGood} that our approximation has far improved error in far less time. Notably, this phenomenon becomes more pronounced as the dimension gets higher; while spending more computation on the Neumann series approximation does noticeably decrease the error for the $N=80, D = 40$ case, we see that as soon as we step into even moderate dimensions ($D$ in the thousands), spending more computation on the Neumann approximation does not noticeably decrease the error. 
In fact, in the three lowest-dimensional experiments here, the dimension is so low that exactly computing $H\inv$ via a Cholesky decomposition is the fastest method.

We also notice that in the $N = 80, D=40$ experiment, $\tildeIJ$ is a better approximation of exact CV than is $\IJ$ for intermediate values of $K$ (i.e.\ some orange dots sit below the large green dot). We note that we have observed this behavior in a variety of synthetic and real-data experiments. We do not currently have further insight into this phenomenon and leave its investigation for future work.
\section{Previous ACV theory}
\label{app:previousTheory}

We briefly review pre-existing theoretical results on the accuracy of ACV. Theoretical results for the accuracy of $\IJ$ are given by \citet{giordano:2018:ourALOO,koh:2019:IJNSBounds,wilson:2020:modelSelectionALOO}. 
\citet{giordano:2018:ourALOO} give a $O(1/N^2)$ error bound for unregularized problems, which \citet[Proposition 2]{stephenson:2020:sparseALOO} extends to regularized prolems; however, in our GLM case here, both results require the covariates and parameter space to be bounded. 
\citet{koh:2019:IJNSBounds} give a similar bound, but require the Hessian to be Lipschitz, and their bounds rely on the inverse of the minimum singular value of $H$, making them unsuited for describing the low rank case of interest here. 
The bounds of \citet{wilson:2020:modelSelectionALOO} are close to our bounds in \cref{prop:appxLowRankAccuracy}. The difference to our work is that \citet{wilson:2020:modelSelectionALOO} consider generic (i.e.\ not just GLM) models, but also require a Lipschitz assumption on the Hessian. 
We specialize to GLMs, avoid the Lipschitz assumption by noting that it only need hold locally, and provide fully computable bounds.

Various theoretical guarantees also exist for the quality of $\NS$ from \cref{NSdef}. \citet{rad:2018:detailedALOO} show that the error $\| \NS - \thetan\|_2$ is $o(1/N)$ as $N \to \infty$ and give conditions under which the error is a much slower $O(1/\sqrt{N})$ as both $N, D \to \infty$ with $N/D$ converging to a constant. 
\citet{beirami:2017:firstALOO} show that the error is $O(1/N^2)$, but require fairly strict assumptions (namely, boundedness of the covariates and parameter space). 
\citet{koh:2019:IJNSBounds, wilson:2020:modelSelectionALOO} provide what seem to be the most interpretable bounds, but, as is the case for $\IJ$, both require a Lipschitz assumption on the Hessian and the results of \citet{koh:2019:IJNSBounds} depend on the lowest singular value of the Hessian.

\section{Proofs from \cref{sec:methodology,sec:accuracy}}
\label{app:proofs}

\subsection{Proving accuracy of $\NS$ and $\IJ$ under exact low-rank data (\cref{prop:exactLowRankAccuracy})}
\label{app:ELRAccuracy}

Here, we prove that, when the covariate matrix is exactly rank $R \ll D$, the accuracy of $\NS$ and $\IJ$ behaves exactly as in a dimension $R \ll D$ problem. Let $X = U\Sigma V$ be the singular value decomposition of $X$, where $\Sigma$ is a diagonal matrix with only $R$ non-zero entries; let $V_{:R} \in \R^{D \times R}$ be the right singular vectors of $X$ corresponding to these $R$ non-zero singular values. 
We define the restricted, $R$-dimensional problem with covariates $\tilde x_n := V_{:R}^T x_n$ as:
\begin{equation}
	\hat\phi := \argmin_{\phi \in \R^R} \frac{1}{N} \sum_{n=1}^N f(\tilde x_n^T \phi) + \frac{\lambda}{2} \n{\phi}_2^2.
\label{app-restrictedProblem}
\end{equation}
Let $\phin$ be the solution to the leave-one-out version of this problem and $\restrictedIJ$ and $\restrictedNS$ the application of \cref{IJdef,NSdef} to this problem. 
We then have the following proposition, which implies the statement of \cref{prop:exactLowRankAccuracy}.
\begin{prop}[Generalization of \cref{prop:exactLowRankAccuracy}]
	The following hold for all datapoints $n$:
	\begin{align*}
		& x_n^T \thetan = \tilde x_n^T \phin \\
		& x_n^T \IJ = \tilde x_n^T \restrictedIJ \\
		& x_n^T \NS = \tilde x_n^T \restrictedNS.
	\end{align*}	
	In particular, $\abs{ x_n^T \NS - x_n^T \thetan} = \abs{ \tilde x_n^T \restrictedNS - \tilde x_n^T \restrictedThetan}$ and $\abs{ x_n^T \IJ - x_n^T \thetan} = \abs{\tilde x_n^T \restrictedIJ - \tilde x_n^T \restrictedThetan}$, as claimed in \cref{prop:exactLowRankAccuracy}.
\end{prop}
\begin{proof}
	First, note that if $\hat\phi$ is an optimum of \cref{app-restrictedProblem}, then $(1/N) \sum_n \dnone V_{:R}^T x_n + \lambda \hat\phi = 0$. 
	As $V_{:R} V_{:R}^T x_n = x_n$, we have that $\hat\theta = V_{:R} \hat\phi$ is optimal for the full, $D$-dimensional, problem. This implies that $\hat\phi = V_{:R}^T \hat\theta$, and thus $x_n^T \hat\theta = \tilde x_n^T \hat\phi$. 
	The same reasoning shows that $x_n^T \thetan = \tilde x_n^T \phin$.
	
	Now, notice that the Hessian of the restricted problem, $H_R$, is given by $H_R  = (1/N) \sum_n V_{:R}^T x_n x_n^T V_{:R} \dntwo + \lambda I_R \implies H_R\inv = V_{:R}^T H\inv V_{:R}$, where the $\dntwo$ are evaluated at $\tilde x_n^T \hat\phi = x_n^T \hat\theta$. 
	Also, the gradients of the restricted problem are given by $\nabla_\phi f(\tilde x_n^T \hat\phi, y_n) = \dnone V_{:R}^T x_n$. Thus the restricted IJ is:
	$$
		\restrictedIJ = \hat\phi + H_R\inv V_{:R}^T x_n \dnone = V_{:R}^T \left(\hat\theta + H\inv x_n \dnone \right) = V_{:R}^T \IJ.
	$$
	Thus, we have $\tilde x_n^T \restrictedIJ = x_n^T V_{:R} V_{:R}^T \IJ = (V_{:R} V_{:R}^T x_n)^T \IJ$. Using $V_{:R} V_{:R}^T x_n = x_n$, we have that $\tilde x_n^T \restrictedIJ = x_n^T \IJ$. The proof that $\tilde x_n^T \restrictedNS = x_n^T \NS$ is identical.
\end{proof}

\subsection{Proving accuracy of $\NS$ and $\IJ$ under ALR data (\cref{prop:appxLowRankAccuracy})}
\label{app:ALRAccuracy}

We will first need a few lemmas relating to how the exact solutions $\thetan$  and $\hat\theta$ vary as we leave datapoints out and move from exactly low-rank to ALR. We start by bounding $\| \hat\theta - \thetan\|_2$; this result and its proof are from \citet[Lemma 16]{wilson:2020:modelSelectionALOO} specialized to our GLM context.
\begin{lm} \label{lm:thetaHatMinusThetan}
	Assume that $\lambda > 0$. Then:
	\begin{equation} \n{\hat\theta - \thetan}_2 \leq \frac{1}{N \lambda} \abs{\dnone} \n{x_n}_2. \end{equation}
\end{lm}
\begin{proof}
	Let $\objn$ be the leave-one-out objective, $\objn(\theta) = (1/N) \sum_{m: \, m \neq n} f(x_m^T \theta, y_m) + (\lambda / 2) \| \theta\|_2^2$.
	As $\objn$ is strongly convex with parameter $\lambda$, we have:
	$$ \lambda \n{\hat\theta - \thetan}_2^2 \leq \iprod{\hat\theta - \thetan}{\nabla \objn (\hat\theta) - \nabla\objn(\thetan)}.$$
	Now, use the fact that $\nabla \objn(\thetan) = \nabla F(\hat\theta) = 0$ and then that $\objn - F = (1/N)f(x_n^T\theta)$ to get:
	\begin{align*}
		&= \iprod{\hat\theta - \thetan}{\nabla \objn(\hat\theta) - \nabla F(\hat\theta)} = \iprod{\hat\theta - \thetan}{\nabla f(x_n^T\hat\theta)} \\
		&\leq \n{\hat\theta - \thetan}_2 \abs{\dnone} \n{x_n}_2.
	\end{align*}
\end{proof}
We will need a bit more notation to discuss the ALR and exactly low-rank versions of the same problem. Suppose we have a $N \times D$ covariate matrix $X$ that is exactly low-rank (ELR) with rows $\xELR \in \R^D$. 
Then, suppose we form some approximately low-rank (ALR) covariate matrix by adding $\eps_n \in \R^D$ to all $x_n$ such that $X\eps_n = 0$ for all $\eps_n$. Let $\xALR$ be the rows of this ALR matrix. Let $\thetaELR$ be the fit with the ELR data and $\thetaALR$ the fit with the ALR data. Finally, define the scalar derivatives:
\begin{align*}
	&\dmoneELR{\theta} := \frac{df(z, y_n)}{dz}\at{z = \iprod{\xELR}{\theta}} \\
	&\dmoneALR{\theta} := \frac{df(z, y_n)}{dz}\at{z = \iprod{\xALR}{\theta}}
\end{align*}
We can now give an upper bound on the difference between the ELR and ALR fits $\| \thetaELR - \thetaALR\|_2$.
Our bound will imply that the $\thetaALR$ is a continuous function of the $\eps_n$, which in turn are continuous functions of the singular values of the ALR covariate matrix.
\begin{lm} \label{lm:lowRankContinuous}
	Assume $\lambda > 0$. We have:
	$$ 
		\n{\thetaELR - \thetaALR}_2 \leq \frac{1}{N \lambda} \n{\sum_{n=1}^N \dmoneELR{\thetaELR} \eps_n}_2
	$$
	In particular, $\thetaALR$ is a continuous function of the $\eps_n$ around $\eps_1, \dots, \eps_N = 0$.
\end{lm}
\begin{proof}
	Denote the ALR objective by $\objALR(\theta) = (1/N) \sum_n f(\xALR^T \theta) + \lambda \| \theta \|_2^2$. Then, via a Taylor expansion of its gradient around $\thetaALR$:
	$$
		\nabla_\theta \objALR(\thetaELR) = \nabla_\theta \objALR(\thetaALR) + \nabla^2_\theta \objALR(\tilde\theta) (\thetaELR - \thetaALR),
	$$
	where $\tilde\theta \in \R^D$ satisfies $\tilde\theta_d = (1-s_d) \theta_{ALR, d} + s_d \theta_{ELR, d}$ for some $s_d \in [0,1]$ for each $d = 1, \dots, D$. Via strong convexity and $\nabla_\theta \objALR(\thetaALR) = 0$, we have:
	$$
		\n{\thetaELR - \thetaALR}_2 \leq \frac{1}{\lambda} \n{\nabla_\theta \objALR(\thetaELR)}_2.
	$$
	Now, note that the gradient on the right hand side of this equation is equal to 
	\begin{equation}
	\nabla_\theta \objALR(\thetaELR) = \frac{1}{N} \sum_{n=1}^N \dmoneALR{\thetaELR} \xELR + \frac{1}{N} \sum_{n=1}^N \dmoneALR{\thetaELR} \eps_n + \lambda \thetaELR . \label{ALRgrad}
	\end{equation}
	By the optimality of $\thetaELR$ for the exactly low-rank problem, we must have that $\iprod{\eps_n}{\thetaELR} = 0$ for all $n$; in particular, this implies that $\iprod{\xELR}{\thetaELR} = \iprod{\xALR}{\thetaELR}$, which in turn implies $\dmoneALR{\thetaELR} = \dmoneELR{\thetaELR}$ for all $n$. Also by the optimality of $\thetaELR$, we have $(1/N) \sum_n \dmoneELR{\thetaELR} x_n + \lambda \thetaELR = 0$. Thus we have that \cref{ALRgrad} reads:
	$$	
		\nabla_\theta \objALR(\thetaELR) = \frac{1}{N} \sum_{n=1}^N \dmoneELR{\thetaELR} \eps_n,
	$$
	which completes the proof.
\end{proof}
We now restate and prove \cref{prop:appxLowRankAccuracy}.
\begin{replemma}{prop:appxLowRankAccuracy} \label{prop:app-appxLowRankAccuracy}
	Assume that $\lambda > 0$ and recall the definition of $L_n$ from \cref{LnDef}. 
	Then, for all $n$:
	\begin{align}
		& \abs{x_n^T \NS - x_n^T \thetan} \leq \frac{L_n}{N^2 \lambda^3} \abs{\dnone}^2 \n{x_n}_2^3 \label{app-appxBoundNS} \\
		& \abs{x_n^T \IJ - x_n^T \thetan} \leq \frac{L_n}{N^2 \lambda^3} \abs{\dnone}^2 \n{x_n}_2^3 + \frac{1}{N^2 \lambda^2} \abs{\dnone} \dntwo \n{x_n}_2^4. \label{app-appxBoundIJ}
	\end{align}
	Furthermore, these bounds continuously decay as the data move from exactly to approximately low rank in that they are continuous in the singular values of $X$.
\end{replemma}
\begin{proof}
	The proof of \cref{app-appxBoundNS,app-appxBoundIJ} strongly resembles the proof of \citet[Lemma 17]{wilson:2020:modelSelectionALOO} specialized to our current context. We first prove \cref{app-appxBoundNS}. We begin by applying the Cauchy-Schwarz inequality to get:
	$$
		\abs{x_n^T \NS - x_n^T \thetan} \leq \n{x_n}_2 \n{\NS - \thetan}_2. 
	$$
	The remainder of our proof focuses on bounding $\| \NS - \thetan \|_2$. Let $\widetilde\objn$ be the second order Taylor expansion of $\objn$ around $\hat\theta$; then $\NS$ is the minimizer of $\widetilde\objn$. 
	By the strong convexity of $\widetilde\objn$:
\begin{align}
	\lambda \n{\thetan - \NS}_2^2 &\leq \iprod{\NS - \thetan}{\nabla \widetilde\objn(\NS) - \nabla\widetilde\objn(\thetan)} \\
	&= \iprod{\NS - \thetan}{\nabla\objn(\thetan) - \nabla\widetilde\objn(\thetan)} \label{thetanNSConvexityBound}
\end{align}
Now the goal is to bound this quantity as the remainder in a Taylor expansion. To this end, define $r(\theta) := \iprod{\thetan - \NS}{\nabla\objn(\theta)}$. 
To apply Taylor's theorem with integral remainder, define $g(t) := r((1-t)\hat\theta + t\thetan)$ for $t \in [0,1]$. Then, by a zeroth order Taylor expansion:
$$ 
	g(1) = g(0) + g'(0) + \int_0^1 \left(g'(s) - g'(0)\right) d s.
$$
Putting in the values of $g$ and its derivatives:
\begin{align*}
	\iprod{\thetan - \NS}{\nabla\objn(\thetan)} &= \iprod{\thetan - \NS}{\nabla\objn(\hat\theta)} + \iprod{\thetan - \NS}{\nabla^2 \objn(\hat\theta) \big( \thetan - \hat\theta \big)} \\
	&+ \int_0^1 \iprod{\thetan - \NS}{ \left(\nabla^2 \objn((1-s)\hat\theta + s\thetan) - \nabla^2 \objn(\hat\theta)\right) \big( \thetan - \hat\theta \big)} d s
\end{align*}
Now, subtracting the first two terms on the right hand side from the left, we get can identify the left with \cref{thetanNSConvexityBound}. Thus, \cref{thetanNSConvexityBound} is equal to:
$$ = \int_0^1 \iprod{\thetan - \NS}{ \left(\nabla^2 \objn((1-s)\hat\theta + s\thetan) - \nabla^2 \objn(\hat\theta)\right) \big( \thetan - \hat\theta \big)} d s.$$
We can upper bound this by taking an absolute value, then applying the triangle inequality and Cauchy-Schwarz to get
\begin{equation}
	 \leq \n{\thetan - \NS}_2 \n{\thetan - \hat\theta} \int_0^1 \n{\left(\nabla^2 \objn((1-s)\hat\theta + s\thetan) - \nabla^2 \objn(\hat\theta)\right)}_{op} d s. \label{taylorRemainderBound}
\end{equation}
Using the fact that, on the line segment $(1-s) \hat\theta + s\thetan$, the $\dntwo$ are lipschitz with constant $C_n$:
$$
	C_n := \max_{s = in [0,1]} \left\lvert \dnthree\left( (1-s) \hat\theta + s\thetan \right) \right\rvert,
$$
 we can upper bound the integrand by:
\begin{align*}
	&\n{\left(\nabla^2 \objn((1-s)\hat\theta + s\thetan) - \nabla^2 \objn(\hat\theta)\right)}_{op}\\
	 &= \frac{1}{N} \n{\sum_{m\neq n} \left(\dntwo\big( (1-s)\hat\theta + s\thetan \big) - \dntwo\big(\hat\theta\big) \right) x_m x_m^T}_{op} \\
	&\leq \frac{C_n \n{\thetan - \hat\theta}_2}{N} \sum_{m\neq n} \n{x_m}_2^2.
\end{align*}
Putting this into \cref{taylorRemainderBound} and using \cref{lm:thetaHatMinusThetan} gives the result \cref{app-appxBoundNS} with $L_n = C_n / N \sum_{m\neq n} \n{x_m}_2^2$.

Now \cref{app-appxBoundIJ} follows from the triangle inequality $\| \IJ - \thetan\|_2 \leq \| \NS - \thetan \|_2 + \| \IJ - NS\|_2$. The bound on $\| \IJ - \NS\|_2$ follows from \citet[Lemma 20]{wilson:2020:modelSelectionALOO}.

Finally, the continuity of the bounds in \cref{app-appxBoundNS,app-appxBoundIJ} follows from \cref{lm:lowRankContinuous}. In particular, the $\dnone, \dntwo$, and $\dnthree$ in both bounds are evaluated at $\thetaALR$, which is shown to be a continuous function of the $\eps_n$ in \cref{lm:lowRankContinuous}. The $\eps_n$ are, in turn, continuous functions of the lower singular values of the covariate matrix. 
\end{proof}

\subsection{Proof of \cref{thm:mainThm}}
\label{app:mainTheoremProof}
\begin{proof}

We first note that the runtime claim is immediate, as \cref{alg:customizedTropp} runs in $O(NDK + K^3)$ time. That the bounds are computable in $O(DK)$ time for each $n$ follows as all derivatives $\dnone$ and $\dntwo$ need only the inner product of $x_n$ and $\hat\theta$, which takes $O(D)$ time. 
Each norm $\|x_n \|_2$ is computable in $O(D)$. For models for which the optimization problem in \cref{prop:MnBound} can be quickly solved -- such as Poisson or logistic regression -- we need only to compute a bound on $\| \thetan - \hat\theta\|_2$, which we can do in $O(D)$ via \cref{lm:thetaHatMinusThetan}. 
The only remaining quantity to compute is the $\eta_n$, which, by \cref{prop:basicErrorBound}, is computed via a projection onto the orthogonal complement of a $K$-dimensional subspace. We can compute this projection in $O(DK)$. Thus our overall runtime is $O(DK)$ per datapoint.

To prove \cref{algIJBnd}, we use the triangle inequality $\abs{x_n^T \tildeIJ - x_n^T \thetan} \leq \abs{x_n^T \IJ - x_n^T \thetan} + \abs{x_n^T \IJ - x_n^T \tildeIJ}$. We upper bound the first term by using \cref{prop:appxLowRankAccuracy}. For the latter, we note that $\abs{x_n^T \IJ - x_n^T \tildeIJ} = \abs{\dnone} \abs{Q_n - \tilde Q_n}$, which we can bound via the $\eta_n$ of \cref{prop:basicErrorBound}. The proof for $\NS$ is similar.
\end{proof}

\subsection{Proof of \cref{cor:lowRank}}
\label{app:corLowRank}

\begin{proof}
	Notice that $\Omega$ from \cref{alg:customizedTropp} captures a rank-$K$ subspace of the column span of $X$. The error bound $\eta_n$ is the norm of $x_n$ projected outside of this subspace divided by $\lambda$. Now, assume that we have $K \geq R$. Then, as the singular values $\sigma_d$ for $d = R+1, \dots, D$ go to zero, the norm of any $x_n$ outside this subspace must also go to zero. Thus $\eta_n$ goes to zero. As $E_n$ is a continuous function of $\eta_n$, we also have $E_n \to 0$.
\end{proof}

\section{Proofs and discussion from \cref{sec:computation}}
\label{app:computation}

\subsection{Proofs}
\label{app:computationProofs}

For convenience, we first restate each claimed result from the main text before giving its proof.
\begin{repproposition}{prop:basicErrorBound} \label{prop:app-basicErrorBound}
	Let $\lambda > 0$ and suppose there is some subspace $\mathcal{B}$ on which $H$ and $\tildeH$ exactly agree. 
	Then $H\inv$ and $\tildeH\inv$ agree exactly on the subspace $\agree := H \mathcal{B}$, and for all $n = 1,\dots, N$:
	\begin{equation} 
		\abs{x_n^T \tildeH\inv x_n - Q_n} \leq \frac{\n{P^\perp_\agree x_n}_2^2}{\lambda},
	\end{equation}
	where $P^\perp_\agree$ denotes projection onto the orthogonal complement of $\mathcal{A}$.
\end{repproposition}
\begin{proof}
	First, if $H$ and $\tildeH$ agree on $\fwdAgree$, then for $\agree = H \fwdAgree = \tildeH \fwdAgree$, we have $H\inv \agree = \fwdAgree = \tildeH\inv \agree$, as claimed. Then:
	\begin{align*}
		\abs{Q_n - x_n^T \tildeH\inv x_n} &= \abs{x_n^T H\inv x_n - x_n^T \tildeH\inv x_n} \\
		&\leq \abs{(P_\agree^\perp x_n) (H\inv - \tildeH\inv)(P_\agree^\perp x_n)} \\
		&\leq \n{P_\agree^\perp x_n}_2^2 \n{H\inv - \tildeH\inv}_{op, \agree^\perp},
	\end{align*}
	where $\n{\cdot}_{op, \agree^\perp}$ is the operator norm of a matrix restricted to the subspace $\agree^\perp$. On this subspace, the action of $\tildeH\inv$ is $1/\lambda$ times the identity, whereas all eigenvalues of $H\inv$ are all between 0 and $1/\lambda$. Thus:
	\begin{align*}
		\n{\tildeH\inv - H\inv }_{op,\agree^\perp} &= \max_{v \in \agree^\perp, \; \n{v}_2 = 1} \left[ v^T \tildeH \inv v - v^T H\inv v \right] \\
		&= \max_{v \in \agree^\perp, \; \n{v}_2 = 1} \left[ \frac{1}{\lambda} - v^T H\inv v \right] \leq \frac{1}{\lambda}.
	\end{align*}
\end{proof}
We next restate and proof \cref{prop:qnBound}.
\begin{repproposition}{prop:qnBound} \label{prop:app-qnBound}
	The $Q_n = x_n^T H\inv x_n$ satisfy
	$
		0 \leq Q_n \leq \|x_n\|_2^2 / (\lambda + \dntwo \| x_n \|_2^2).
	$
	Furthermore, letting $\tildeqn := \mathrm{min} \{ x_n^T \tildeH\inv x_n, \, \n{x_n}_2^2 / (\lambda + \dntwo \n{x_n}_2^2) \}$, we have the error bound
	\begin{equation}
		\abs{\tildeqn - Q_n} \leq \min\set{\frac{\n{P^\perp_\agree x_n}_2^2}{\lambda}, \frac{\|x_n\|_2^2}{\lambda + \dntwo \|x_n\|_2^2}}.
	\end{equation}
\end{repproposition}
\begin{proof}
	Let $b_n := \sqrt{\dntwo} x_n$. Let $\set{v_d}_{d=1}^D$ be the eigenvectors of $H$ with eigenvalues $\set{\gamma_d + \lambda}_{d=1}^D$ with $\gamma_1 \geq \gamma_2 \geq \dots \geq \gamma_D$. The quantity $b_n^T H\inv b_n$ is maximized if $b_n$ is parallel to $v_D$; in this case, $b_n^T H\inv b_n = \n{b_n}_2^2 / (\gamma_D + \lambda)$. However, if $b_n$ is parallel to $v_D$, it must be that $\gamma_D \geq \n{b_n}_2^2$. Thus, $b_n^T H\inv b_n \leq \n{b_n}_2^2 / (\n{b_n}_2^2 + \lambda)$. Dividing by $\dntwo$ gives that $Q_n = x_n^T H\inv x_n$ satisfies:
	$$ 0 \leq Q_n \leq \frac{\n{x_n}_2^2}{\lambda + \dntwo \n{x_n}_2^2}.$$
	If we estimate $Q_n$ by the minimum of this upper bound and $x_n^T \tildeH\inv x_n$, the error bound from \cref{prop:basicErrorBound} implies the error bound claimed here.
\end{proof}

\subsection{Relation of \cref{alg:customizedTropp} to techniques from randomized linear algebra}
\label{app:tropp}

As noted, our \cref{alg:customizedTropp} bears a resemblance to techniques from the randomized numerical linear algebra literature. Indeed, our inspiration for \cref{alg:customizedTropp} was the work of \citet{tropp:2017:psdRandomizedApproximation}. 
\citet{tropp:2017:psdRandomizedApproximation} propose a method to find a randomized top-$K$ eigendecomposition of a positive-semidefinite matrix $B$. Their method follows the basic steps of (1) produce a random orthonormal matrix $\Omega \in \R^{D \times (S+K)}$, where $S \geq 0$ is an \emph{oversampling} parameter to ensure the stability of the estimated eigendecomposition, (2) compute the Nystr\"{o}m approximation of $B_{nys} \approx B$ using $\Omega$, and (3) compute the eigendecomposition of $B_{nys}$ and throw away the lowest $S$ eigenvalues. 

Our \cref{alg:customizedTropp} can be seen as using this method of \citet{tropp:2017:psdRandomizedApproximation} to obtain a rank-$K$ decomposition of the matrix $B = (1/N) \sum_n \dntwo x_n x_n^T$ with specific choices of $S$ and $\Omega$. First, we notice that $S = 0$ (i.e., no oversampling) is optimal in our application -- the error bound of \cref{prop:basicErrorBound} decreases as the size of the subspace $\agree$ increases. As $S > 0$ only decreases the size of this subspace, we see that our specific application is only hurt by oversampling. Next, while \citet{tropp:2017:psdRandomizedApproximation} recommend completely random matrices $\Omega$ for generic applications (e.g., the entries of $\Omega$ are i.i.d.\ $\normaldist(0,1)$), we note that the results of \cref{prop:optimalAgree} suggest that we can improve upon this choice. With the optimal choice of $S = 0$, we note that $\agree = H \Omega$. In this case, \cref{prop:optimalAgree} implies it is optimal to set $\Omega = H\inv V_{:K}$, where $V_{:K}$ are the top-$K$ right singular vectors of $X$. \cref{alg:customizedTropp} provides an approximation to this optimal choice.

We illustrate the various possible choices of $\Omega$, including i.i.d.\ $\normaldist (0,1)$, in \cref{appxHExperiments}. We generate a synthetic Poisson regression problem with covariates $x_{nd} \overset{i.i.d.}{\sim} \normaldist (0,1)$ and $y_n \sim \mathrm{Poisson}(e^{x_n^T \theta^*})$, where $\theta^* \in \R^D$ is a true parameter with i.i.d.\ $\normaldist (0,1)$ entries. We generate a dataset of size $N = 200$. The covariates have dimension $D = 150$ but are of rank 50. We compute $\widetilde IJ$ for various settings of $K$ and $\Omega$, as shown in \cref{appxHExperiments}. As suggested by the above discussion, we use no oversampling (i.e., $S = 0$). On the left, we see that using a diagonal approximation to $H\inv$ and a single subspace iteration gives a good approximation to the optimal setting of $\Omega$. On the right, we see the improvement made by use of the upper bound on $x_n^T H\inv x_n$ from \cref{prop:qnBound}.
\begin{figure}
\centering
\begin{tabular}{cc}
	\includegraphics[scale=0.45]{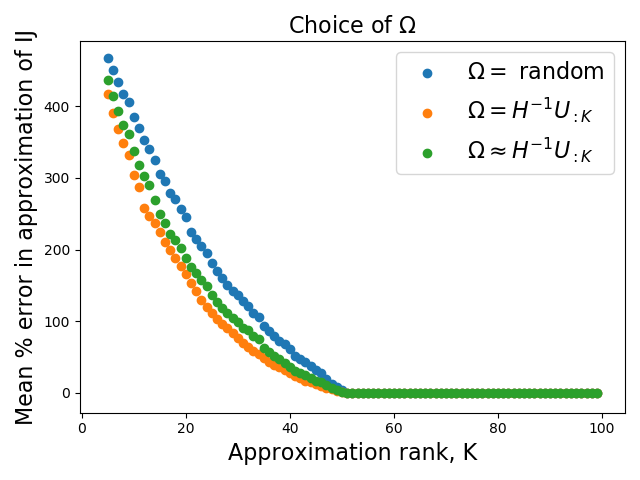} &
	\includegraphics[scale=0.45]{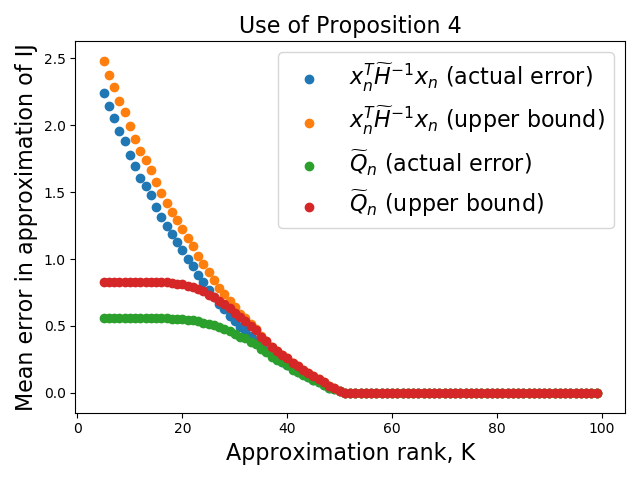}
\end{tabular}
\caption{Quality of approximation of $\IJ$ on a synthetic Poisson regression problem using the methods from \cref{sec:computation}. (\emph{Left}): We show three options for the choice of the matrix $\Omega$. Blue shows the choice of $\Omega$ having orthonormal columns selected uniformly at random, orange the optimal choice of $\Omega$ from \cref{prop:optimalAgree}, and green our approximation to this optimal choice. Percent error $\abs{\IJ - \widetilde\IJ} / \abs{\IJ}$ is reported to give a sense of scale. (\emph{Right}): Importance of \cref{prop:qnBound} for approximating $Q_n$. We show two approximations along with our upper bounds on their error: (1) $Q_n \approx x_n^T \tildeH\inv x_n$ and (2) our recommended $Q_n \approx \tildeqn$ from \cref{prop:qnBound}. We report absolute error $\abs{\IJ - \widetilde\IJ}$ so that both actual and estimated error can be plotted.}
\label{appxHExperiments}
\end{figure}

\subsection{Implementation of \cref{alg:customizedTropp}}
\label{app:computationImplementation}

As noted by \citet{tropp:2017:psdRandomizedApproximation}, finding the decomposition of $B$ in \cref{alg:customizedTropp} as-written can result in numerical issues. Instead, \citet{tropp:2017:psdRandomizedApproximation} present a numerically stable version which we use in our experiments. For completeness, we state this implementation here, which relies on computing the Nystr\"{o}m approximation of the shifted matrix $B_\nu = B + \nu I_D$, for some small $\nu > 0$:
\begin{enumerate}
	\item Construct the shifted matrix sketch $G_\nu := B\Omega + \nu\Omega$.
	\item Form $C = \Omega^T G_\nu$.
	\item Compute the Cholesky decomposition $C = \Gamma \Gamma^T$.
	\item Compute $E = G_\nu \Gamma\inv$.
	\item Compute the SVD $E = U \Sigma V^T$.
	\item Return $U$ and $\Sigma^2 - \nu I$ as the approximate eigenvectors and eigenvalues of $B$.
\end{enumerate}

\section{Error bound experiments}
\label{app:errorBoundExps}

\begin{figure}
\label{app-boundExperiments}
\centering
\begin{tabular}{cc}
	\includegraphics[scale=0.45]{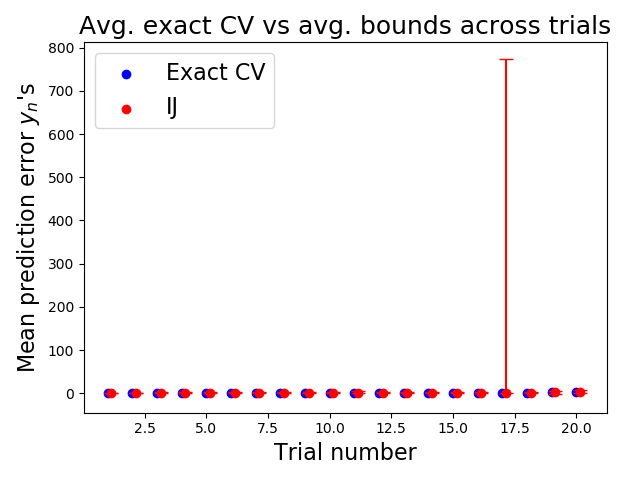} &
	\includegraphics[scale=0.45]{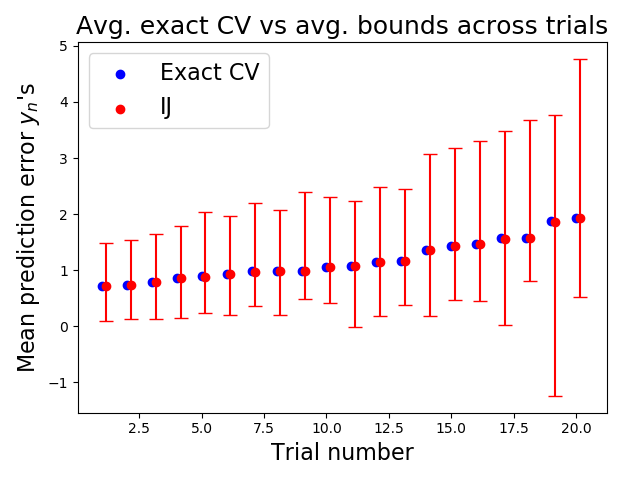}
\end{tabular}
\caption{Experiments from \cref{app:errorBoundExps}. (\emph{Left}): Error bounds can be vacuously large; for trial number 16, our bound exceeds exact CV by two orders of magnitude. (\emph{Right}): By computing our bound for all $n$ and re-running exact CV for the two largest values, we obtain estimates that are much closer to exact CV.}
\end{figure}

Here, we provide more details on our investigation of the error bounds of \cref{thm:mainThm} from \cref{fig:errorBounds}. In \cref{sec:experiments}, we showed that, over five randomly generated synthetic datasets, our error bound on $x_n^T \tildeIJ$ implies upper bounds on out-of-sample error that are reasonably tight. However, we noted that these bounds can occasionally be vacuously loose. On the left of \cref{app-boundExperiments}, we show this is the case by repeating the experiment in \cref{fig:errorBounds} for an additional fifteen trials. While most trials have similar behavior to the first five, trial 16 finds an upper bound of the out-of-sample error that is too loose by two orders of magnitude. However, we note that this behavior is mostly due to two offending points $n$. Indeed, on the right of \cref{app-boundExperiments}, we show the same results having replaced the two largest bound values with those from exact CV.
\section{Real data experiments}
\label{app:realDataExperiments}

Here we provide more details about the three real datasets used in \cref{sec:experiments}.
\begin{enumerate}
	\item The \texttt{p53} dataset is from \citet{p53_1,p53_2,p53_3}. The full dataset contains $D = 5{,}408$ features describing attributes of mutated p53 proteins. The task is to classify each protein as either ``transcriptionally competent'' or inactive. To keep the dimension high relative to the number of observations $N$, we subsampled $N = 8{,}000$ datapoints uniformly at random for our experiments here. We fix $K = 500$ to compute $\tildeqn$ for both $x_n^T \tildeIJ$ and $x_n^T \tildeNS$.
	\item The \texttt{rcv1} dataset is from \cite{rcv1}. The full dataset is of size $N = 20{,}242$ and $D = 47{,}236$. Each datapoint corresponds to a Reuters news article given one of four labels according to its subject: ``Corporate/Industrial,'' ``Economics,'' ``Government/Social,'' and ``Markets.'' We use a pre-processed binarized version from \url{https://www.csie.ntu.edu.tw/~cjlin/libsvmtools/datasets/binary.html}, which combines the first two categories into a ``positive'' label and the latter two into a ``negative'' label. We found running CV on the full dataset to be too computationally intensive, and instead formed a smaller dataset of size $N = D = 20{,}000$. The data matrix is highly sparse, so we chose our $20{,}000$ dimensions by selecting the most common (i.e., least sparse) features. We then chose $N = 20{,}000$ datapoints by subsampling uniformly at random. We fix $K = 1{,}000$ in this experiment.
	
	\item The \texttt{blog} dataset is from \citet{buza:2014:blogFeedback}. The base dataset contains $D = 280$ features about $N = 52{,}397$ blogs. Each feature represents a statistic about web traffic to the given blog over a 72 hour period. The task is to predict the number of unique visitors to the blog in the subsequent 24 hour period. We first generate a larger dataset by considering all possible pairwise features $x_{nd_1} x_{nd_2}$ for $d_1, d_2 \in \{1, \dots, D\}$. The resulting problem has too high $N$ and $D$ to run exact CV on in a reasonable amount of time, so we again subsample to $N = 20{,}000$ and $D = 20{,}280$. We again choose the $20{,}000$ least sparse parwise features and then add in the original 280 features. Finally, we choose our $20{,}000$ datapoints uniformly at random.
\end{enumerate}
\end{document}